%% file: main_CRC_arxiv.tex
\documentclass{article}

\usepackage[main, preprint]{neurips_2026}
\makeatletter
\renewcommand{\@noticestring}{%
  Accepted at the 40th Conference on Neural Information Processing Systems (NeurIPS 2026).%
}
\makeatother

\usepackage[utf8]{inputenc}
\usepackage[T1]{fontenc}
\usepackage{url}
\usepackage{amsfonts}
\usepackage{nicefrac}
\usepackage{xcolor}

\usepackage{graphicx}
\usepackage{booktabs}
\usepackage[breaklinks=true]{hyperref}
\hypersetup{
  pdftitle={Robust Residual Correction via Selective Deployment for Time Series Forecasting},
  pdfauthor={Jianxiang Xie and Yuncheng Hua and Mingyue Cheng and Flora D. Salim and Hao Xue},
  colorlinks=true,
  linkcolor=blue,
  citecolor=blue,
  urlcolor=blue,
}
\usepackage{microtype}
\usepackage{amsmath}
\usepackage{amssymb}
\usepackage{mathtools}
\usepackage{amsthm}
\usepackage{multirow}
\usepackage{siunitx}
\usepackage[capitalize,noabbrev]{cleveref}
\usepackage{tabularx}
\usepackage{ragged2e}
\usepackage{placeins}
\usepackage{dblfloatfix}
\usepackage{subcaption}
\usepackage{enumitem}

\theoremstyle{plain}
\newtheorem{theorem}{Theorem}[section]
\newtheorem{proposition}[theorem]{Proposition}

\theoremstyle{definition}

\newtheorem{assumption}[theorem]{Assumption}

\theoremstyle{remark}
\newtheorem{remark}[theorem]{Remark}

\title{Robust Residual Correction via Selective Deployment for Time Series Forecasting}

\author{%
  Jianxiang Xie$^{1}$ \quad
  Yuncheng Hua$^{1,4}$ \quad
  Mingyue Cheng$^{2}$ \quad
  Flora D.\ Salim$^{1,4}$ \quad
  Hao Xue$^{3,4}$\thanks{Corresponding author: Hao Xue (haoxue@hkust-gz.edu.cn).} \\
  $^{1}$School of Computer Science and Engineering, University of New South Wales, Sydney, Australia \\
  $^{2}$State Key Laboratory of Cognitive Intelligence, School of Computer Science and Technology,\\
  University of Science and Technology of China, Hefei, China \\
  $^{3}$The Hong Kong University of Science and Technology (Guangzhou) \\
  $^{4}$Data Cruiser
}

\begin{document}
\raggedbottom
\maketitle

\begin{abstract}
Post-hoc residual correction can reduce forecasting error, but {always-on} updates often fail to generalize: when residual signals are weak or noisy, expressive correctors can overfit and even {increase} the overall forecasting errors.
We introduce \textbf{CRC}, a robust residual correction framework that couples a hybrid corrector with a validation-calibrated \emph{selective deployment} rule.
CRC proposes corrections via a conservative ridge ``floor'' plus conditional nonlinear refinement, but deploys them only when held-out evidence indicates net error reduction over a frozen baseline forecaster.
CRC treats \emph{when} to revise a forecast as a validation-calibrated deployment decision, not only how to fit residuals.
Across standard long-sequence benchmarks and multiple backbones, CRC improves horizon-averaged MSE and MAE in most dataset--backbone pairs, with particularly large gains on complex datasets (e.g., over 20\% relative MSE reduction on Traffic).
We additionally report the non-degradation rate (NDR) as a stability diagnostic of harmful updates and ablate the validation-calibrated ``firewall'' that stabilizes MSE gains.
\textit{Code:} \url{https://github.com/Xavian1-oss/CRC.git}.
\end{abstract}

\section{Introduction}
\label{sec:intro}

Multivariate time series forecasting underpins decisions in transportation, energy, finance, and urban systems.
Modern pipelines produce vector-valued trajectories across many horizons. Accuracy is judged not only globally but channel-by-channel and step-by-step when planners route traffic, schedule power, or monitor anomalies.
Deep learning has improved average accuracy~\cite{lim2021time,wen2022transformers,hochreiter1997long,vaswani2017attention}, yet prediction errors typically vary across channels and horizons.
Where forecasts gate operations or downstream risk assessment, localized failures matter even when headline metrics improve.

Because retraining large forecasters is costly, \emph{post-hoc residual correction} is an appealing add-on: one fits an auxiliary module on baseline errors while keeping the backbone frozen~\cite{goel2017r2n2,wolpert1992stacked}.
The usual recipe regresses residual evidence against targets and applies the resulting update \emph{by default} at test time.
That \emph{always-on} regime is fragile: under weak residual structure relative to noise or misspecification, expressive maps can chase idiosyncratic fluctuations and \emph{increase} test mean squared error (MSE) relative to the baseline---even when training loss looks favorable.
This tension is the \emph{Corrector's Dilemma}: bias that could be reduced in aggregate must still be shielded from noise amplification at deployment time, so the operative question includes \emph{whether} and \emph{when} to apply an update, not only how to parameterize one.
Importantly, the difficulty is not resolved by fitting residuals more aggressively on the training split; without a deployment gate, stronger correctors can amplify the mismatch between in-sample fit and held-out error.

We study selective deployment under deterministic point forecasting from a frozen backbone---the evaluation protocol adopted by standard multivariate forecasting benchmarks~\cite{zhou2021informer,wu2021autoformer}---with squared error summarized at horizon and variable granularity.
\textbf{CRC} consumes the historical window and baseline prediction, builds direction-aware residual features via a ridge ``floor'' and conditional nonlinear refinement, and passes proposals through a deployment rule calibrated on validation and frozen at test time; it selects among baseline and refinement paths using held-out evidence (\emph{without} oracle access to test residuals).
Our goal is improved test accuracy under the benchmark protocol and honest diagnostics when updates fail locally instead of distribution-shift guarantees beyond what validation calibration reasonably supports (limitations are discussed in Section~\ref{sec:discussion}).

\paragraph{Contributions.}
\begin{itemize}[leftmargin=*]
    \item \textbf{Perspective.}
    Residual correction is framed as deciding \emph{deployment} rather than maximizing in-sample residual fit alone: expressive updates carry generalization risk when residual evidence is thin, motivating a validation-calibrated veto over always-on refinement (complementary to coverage-based or abstention-focused methods; see Section~\ref{sec:related}).

    \item \textbf{Method.}
    CRC couples a hybrid corrector with a granular selector: direction-aware features, ridge-backed linear recovery, guarded nonlinear refinement, and per variable--horizon deployment choices calibrated on validation and frozen for testing.
    The deployment rule composes simple, interpretable filters (direction consistency, magnitude clipping, local path selection, and shrink-to-base margins) so that refinements are released only where validation supports net benefit.

    \item \textbf{Empirics and diagnostics.}
    Across seven benchmarks and four backbones, CRC improves horizon-averaged MSE and MAE in most dataset--backbone combinations, with larger margins where cross-channel residual structure appears informative.
    We report non-degradation rate (NDR) alongside error metrics to quantify how often corrections worsen local predictions, and use ablations to isolate the deployment ``firewall'' relative to architectural choices.
\end{itemize}

\paragraph{Granularity.}
Operations often care both about pooled error and about individual variable--horizon behavior.
Aggregate MSE may improve while silently degrading forecasts on a subset of sensors or future steps (precisely where practitioners notice instability or incur uneven costs).
Our non-degradation rate (NDR) estimates how frequently a correction leaves each $(i,h)$ unit no worse than the baseline (ties allowed), exposing harmful updates that averages can obscure.
CRC evaluates and gates deployment at the same $(i,h)$ units tracked by these diagnostics, rather than trusting a single pooled residual regression to authorize updates everywhere.

\section{Related Work}
\label{sec:related}

Most post-hoc pipelines emphasize \emph{how} to model a residual map $\Delta(\cdot)$ and then deploy $\hat{Y}^{\mathrm{base}}+\Delta$ by default (\emph{always-on} correction), which can harm units that are already accurate when residuals are noisy.
Residual modules and instance-aware revisions~\cite{goel2017r2n2,liu2025improving} largely pursue average-error reduction on frozen forecasters.
CRC adds a deploy-time layer $\mathcal{I}(\cdot)$ \emph{calibrated on validation}: it chooses whether to deploy a refinement at fine granularity, complementary to stacking another regressor wholesale (Section~\ref{sec:method}), and plugs into standard long-horizon backbones (e.g., DLinear, PatchTST, TimesNet~\cite{zeng2023transformers,nie2022time,wu2022timesnet}).

Instance-aware post-hoc revision (PIR~\cite{liu2025improving}) targets difficult instances for retrieval-driven repair. CRC instead stabilizes pointwise error by \emph{gating deployment} so harmful updates are suppressed rather than universally applied.
Test-time adaptation and forecaster-selection lines such as TAFAS~\cite{kim2025battling}, PETSA~\cite{medeiros2025accurate}, FAC~\cite{wang2026towards}, SOLID~\cite{chen2024calibration}, and FFORMA~\cite{montero2020fforma} typically update or choose models under test-time feedback.
CRC keeps the backbone frozen, uses no test labels, and only decides whether to release an additive residual correction; NDR is reported as a local non-degradation check.
$\delta$-Adapter~\cite{liang2026forecast} and related adapters illustrate safety-motivated post-hoc adaptation. CRC focuses narrowly on deploying (or withholding) additive residual corrections tied to frozen $\hat{Y}^{\mathrm{base}}$.

Reliability-focused learning spans conformal prediction~\cite{shafer2008tutorial}, selective prediction~\cite{geifman2019selectivenet}, and robust objectives~\cite{mohajerin2018data}; these emphasize coverage/sets or worst-case risk rather than selectively releasing point refinements after $\hat{Y}^{\mathrm{base}}$ is fixed.
Sequential conformal and calibration methods~\cite{stankeviciute2021conformal,xu2021conformal,xu2023sequential} likewise differ in target from \emph{when-to-deploy} corrections aimed at stabilizing deterministic MSE.
Neural basis-expansion architectures such as N-BEATS~\cite{oreshkin2019n} optimize accuracy but rarely attach an explicit veto rule for harmful residual updates at deployment time.
Structure-aware representations~\cite{brouillard2020differentiable,scholkopf2021toward,runge2017detecting} inform inductive bias; we use directional pairwise evidence without claiming causal discovery.

\section{Method}
\label{sec:method}

Throughout this section, \emph{degradation} refers specifically to correction-induced performance deterioration relative to the baseline forecast.
The exposition follows Fig.~\ref{fig:architecture}: oracle and features (\cref{sec:selective_intervention,sec:encoder,sec:hybrid_corrector}), then deployment (\cref{sec:deployment_rule}).

We formulate residual correction as a \emph{selective deployment problem}: CRC seeks lower held-out MSE by deploying updates only where validation indicates net benefit, a calibration \emph{design goal}, not an unconditional guarantee at test time (Appendix~A).
The centerpiece is thus both a correction map $\Delta$ and a frozen deployment rule $\mathcal{I}(\cdot)$ that decides \emph{whether} each proposed update is applied.
Given a historical sequence
$W \in \mathbb{R}^{B \times P \times N}$,
where $B$ is the batch size, $P$ is the look-back window, and $N$ is the number of variables,
together with a baseline forecast
$\hat{Y}^{\mathrm{base}} \in \mathbb{R}^{B \times H \times N}$ for horizon $H$,
the goal is not to apply correction indiscriminately, but to deploy it \emph{only when correction is expected to be beneficial}.
Accordingly, CRC is designed to answer two coupled questions:
(i) \emph{what correction should be proposed}, and
(ii) \emph{when should that correction be applied}?

We formalize the corrected forecast as
\begin{equation}
\hat{Y} = \hat{Y}^{\mathrm{base}} + \mathcal{I}(\Delta),
\label{eq:overall_intervention}
\end{equation}
where $\Delta$ is a candidate residual correction and $\mathcal{I}(\cdot)$ denotes a \emph{single} validation-calibrated deployment mapping, frozen on validation and applied identically at test time.
At the same variable--horizon granularity used by our stability diagnostic (NDR, Section~\ref{sec:experiments}), $\mathcal{I}$ decides how much of each proposed update is released instead of concatenating unrelated knobs.
Operationally, $\mathcal{I}$ is an empirical surrogate for a \emph{risk-aware selector}: deploy a proposed $\Delta$ when validation evidence supports that the expected MSE benefit outweighs the risk of noise amplification under the firewall constraints.
Our objective is therefore not aggressive residual fitting, but robust refinement: deploy corrections only when validation evidence indicates a favorable improvement--risk tradeoff, so as to reduce test MSE while limiting harmful updates.
\paragraph{Ideal objective (for intuition).}
Here $\text{Error}(\cdot)$ denotes the per-unit loss used to judge whether a correction is beneficial (e.g., squared error for MSE or absolute error for MAE) at the deployment granularity.
\begin{equation}
\text{Error}(\hat{Y}) \le \text{Error}(\hat{Y}^{\mathrm{base}})
\label{eq:nondegrade_obj}
\end{equation}
whenever the deployment rule approves the correction.
This inequality is an \emph{idealized} design principle (not a guarantee). CRC approximates it by calibrating and freezing the deployment rule on validation.

\paragraph{Overview.}
CRC follows a robust-by-design pipeline (Fig.~\ref{fig:architecture}) with three components:
(1) a direction-aware residual representation that exposes recoverable structure,
(2) a conservative--conditional corrector that proposes candidate updates, and
(3) a validation-calibrated deployment rule that determines whether and how those updates are deployed.

\begin{figure*}[t!]
\centering
\makebox[\textwidth][c]{%
\includegraphics[width=\textwidth]{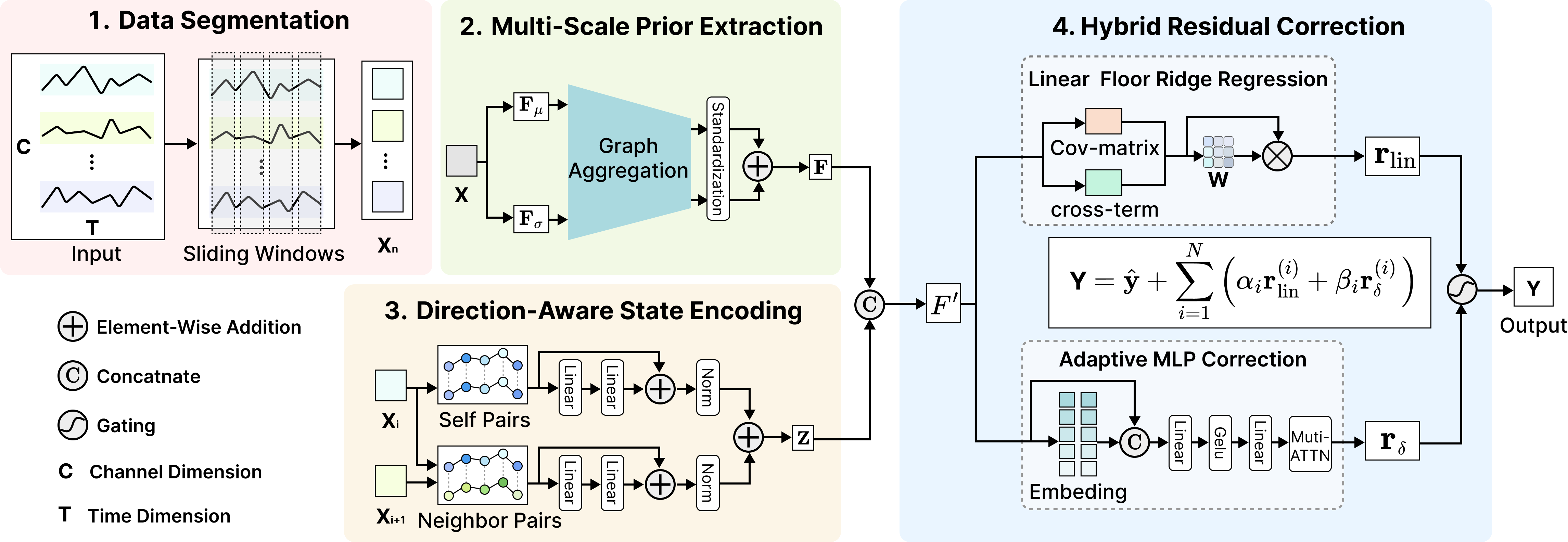}%
}
\caption{Overview of the CRC framework.
CRC first extracts structured residual representations, then proposes conservative and conditional correction components, and finally applies a validation-calibrated deployment rule that deploys corrections only when held-out evidence supports them.}
\label{fig:architecture}
\end{figure*}

\subsection{Residual Correction as Selective Deployment}
\label{sec:selective_intervention}

Residuals contain both \emph{recoverable structure} and \emph{irreducible noise}, expressive correctors may harm accuracies by chasing spurious fluctuations, that is, the \textbf{Corrector's Dilemma} of \Cref{sec:intro}.

CRC adopts a selective deployment view.
Instead of always applying a learned correction,
we first generate a \emph{candidate update} and then apply it only under a validation-calibrated deployment rule.
Formally, let $\Delta_i$ denote the candidate correction for variable $i$.
An ideal corrector would follow the oracle rule
\begin{equation}
\hat{Y}_i =
\begin{cases}
\hat{Y}^{\mathrm{base}}_i + \Delta_i, & \text{if } \text{Error}(\hat{Y}^{\mathrm{base}}_i + \Delta_i) \le \text{Error}(\hat{Y}^{\mathrm{base}}_i), \\
\hat{Y}^{\mathrm{base}}_i, & \text{otherwise}.
\end{cases}
\label{eq:oracle_rule}
\end{equation}
Since this oracle is unavailable at deployment, CRC approximates it through a validation-calibrated deployment rule.
This perspective turns residual correction from an unconstrained regression problem into a \emph{decision-aware refinement problem}.

\subsection{Structured Residual Representation}
\label{sec:encoder}

A selective deployment rule is only meaningful if the candidate correction is grounded in structured residual evidence.
Directly modeling the raw residual
$R = Y - \hat{Y}^{\mathrm{base}}$
is unreliable, because it entangles correctable bias with random noise.
CRC therefore first constructs a direction-aware representation designed to expose residual structure that is stable enough to support conservative correction.
The encoder feeds both the hybrid corrector and the deployment rule that may veto harmful updates.

Given prior adjacency $A \in \{0,1\}^{N \times N}$ ($A_{ij}=1$ if $j$ may influence $i$), a shared module $f_{\text{pair}}$ processes self pairs and cross pairs.
We write $U_i$ for the learned pairwise interaction embedding (distinct from the raw residual $R$ above):
\begin{equation}
(U_i, \alpha_{i \to i}) = f_{\text{pair}}(W_{:,:,i}, W_{:,:,i}),
\label{eq:self_dyn}
\end{equation}
\begin{equation}
(U_j, \alpha_{j \to i}) = f_{\text{pair}}(W_{:,:,i}, W_{:,:,j}), \quad A_{ij}=1.
\label{eq:cross_dyn_new}
\end{equation}
Gates $\alpha$ are direction-aware (via $\tanh$) but are not intended for causal identification.
We aggregate
\begin{equation}
Z_i =
\frac{1}{|\mathcal{N}(i)| + \alpha_{i \to i,\text{const}}}
\left(
\alpha_{i \to i} U_i + \sum_{j:A_{ij}=1} \alpha_{j \to i} U_j
\right),
\label{eq:gated_agg_new}
\end{equation}
where $\alpha_{i \to i,\text{const}}>0$ is a fixed stabilizing constant (set to $1$ in our implementation) that prevents division by zero when $\mathcal{N}(i)$ is empty and keeps the normalization well-defined.
We then concatenate multi-scale statistical priors $F$ (means, trends, correlations, band energy; neighborhood pooled with $A_{\text{norm}}$; see Appendix~\ref{app:feature_priors}) to obtain
\begin{equation}
Z_{\text{aug}} = [Z, F].
\label{eq:zaug}
\end{equation}

\paragraph{Constructing the adjacency used by $f_{\text{pair}}$.}
The binary adjacency $A$ specifies which ordered pairs $(j\!\to\! i)$ are eligible for cross-node messaging.
In our implementation, $A$ is built once from training data by correlation-based $K$-nearest neighbors over variables ($K{=}5$; see \Cref{app:training}), then fixed for validation and testing.
This yields a sparse prior over interactions without assuming a known causal graph; the learnable gates $\alpha$ modulate which edges are active for a given window.

\subsection{Conservative--Conditional Correction}
\label{sec:hybrid_corrector}

Given the structured features $Z_{\text{aug}}$, CRC constructs a candidate correction by decomposing it into
a \emph{conservative component} and a \emph{conditional refinement component}.
This decomposition reflects a key design principle:
stable, linearly explainable residual structure should be recovered by default,
while more expressive nonlinear correction should only be trusted conditionally.

For each variable $i$, CRC defines the candidate correction as
\begin{equation}
\Delta_i = \Delta_i^{\mathrm{base}} + \Delta_i^{\mathrm{ref}},
\label{eq:delta_decomp}
\end{equation}
where $\Delta_i^{\mathrm{base}}$ is a conservative baseline recovery term and $\Delta_i^{\mathrm{ref}}$ is a higher-order refinement term.

\paragraph{Conservative baseline recovery.}
CRC uses ridge regression~\cite{hoerl1970ridge} to recover stable, linearly explainable bias:
\begin{equation}
\boldsymbol{\beta}_i
=
\arg\min_{\boldsymbol{\beta}}
\left(
\|\mathbf{r}_i - Z_{\text{aug},i}\boldsymbol{\beta}\|_2^2
+ \lambda \|\boldsymbol{\beta}\|_2^2
\right),
\label{eq:ridge_new}
\end{equation}
where
$\mathbf{r}_i = \mathbf{Y}_i - \hat{\mathbf{Y}}^{\mathrm{base}}_i$
is the residual of the baseline forecast.
The resulting conservative correction is
\begin{equation}
\Delta_i^{\mathrm{base}} = Z_{\text{aug},i}\boldsymbol{\beta}_i.
\label{eq:delta_base}
\end{equation}
This term acts as a stable recovery floor, capturing residual structure that is both interpretable and robust.

\paragraph{Conditional refinement.}
To model higher-order residual patterns beyond the linear floor, CRC employs a lightweight MLP:
\begin{equation}
\Delta_i^{\mathrm{ref}} = f_{\theta}(Z_{\text{aug},i}).
\label{eq:delta_ref}
\end{equation}
This refinement is expressive but potentially risky, since nonlinear updates are more vulnerable to overfitting noise.
Accordingly, $\Delta_i^{\mathrm{ref}}$ is not deployed directly. Instead, it is passed to the deployment rule described next.

\subsection{Validation-Calibrated Deployment Rule}
\label{sec:intervention_rule}
\label{sec:deployment_rule}

The role of $\mathcal{I}(\cdot)$ is to approximate the ideal oracle in Eq.~\eqref{eq:oracle_rule}.
We treat $\mathcal{I}$ as \emph{one} test-frozen deployment policy: it composes (i) \emph{stabilizing transforms} of the nonlinear proposal (direction checks and magnitude clipping) with (ii) a \emph{baseline-inclusive} discrete path choice at each $(i,h)$ unit, and (iii) a validation-level \emph{margin} that defaults to the baseline when estimated improvements are too small to be reliable.
This factorization targets complementary failure modes (sign flips, heavy-tailed updates, cross-unit heterogeneity, and selection noise) while keeping a single object (i.e., the released update) well defined at the same granularity as NDR.

\subsubsection{Factorization of the deployment rule (robust deployment)}

The following items are the operational steps that instantiate the same $\mathcal{I}$ (each detailed at training/validation time). Clipping thresholds and deployment weights are the primary algebraic objects:
\begin{itemize}[leftmargin=*]
  \item \textbf{Direction consistency (stabilize):} accept nonlinear refinement only when aligned with the post-linear residual direction (suppresses harmful sign flips).
  \item \textbf{Magnitude control (stabilize):} clip candidates with validation-calibrated thresholds~$\tau$:
\begin{equation}
\Delta^{\mathrm{clip}} = \mathrm{clip}(\Delta, -\tau, \tau).
\label{eq:clip_new}
\end{equation}
  Here $\Delta$ is the gated candidate entering clipping. Appendix~A uses an equivalent linear-floor plus clipped-nonlinear form for analysis.
  \item \textbf{Local path selection (discrete, per $(i,h)$):} freeze whichever of linear-only vs.\ hybrid correction wins on validation at each evaluation unit.
  \item \textbf{Global validation gating (margin):} deploy only if validation gains exceed a margin; otherwise revert to the baseline channel.
\end{itemize}
The deployed forecast combines baseline and (possibly zeroed) corrections:
\begin{equation}
\hat{Y} =
\hat{Y}^{\mathrm{base}}
+ w_1 \Delta^{\mathrm{base}}
+ w_2 \Delta^{\mathrm{clip}},
\label{eq:final_new}
\end{equation}
where $(w_1,w_2)$ are chosen on validation per unit (including zeros when a path is disabled), consistent with selecting among candidates that include the baseline (Proposition~\ref{prop:val-safe}).
Thresholds~$\tau$, local choices, and weights are fit on validation and \emph{frozen} for test (see Appendix~\ref{app:training}).

\subsubsection{Interpretation as Oracle Approximation}

The above constraints should not be viewed as isolated heuristics.
Instead, they compose the single mapping $\mathcal{I}$ in Eq.~\eqref{eq:overall_intervention} and provide \emph{design-motivated} approximations of the unavailable oracle rule in Eq.~\eqref{eq:oracle_rule} along four axes (direction, magnitude, locality, and validation reliability). We do not claim these steps are the exact solution of a single constrained optimization problem; rather, they instantiate selective deployment by coupling a candidate update with a frozen, validation-calibrated decision. Appendix~A provides supporting statements under explicit assumptions.

\subsection{Complexity and Deployment}
\label{sec:complexity}

Training cost is dominated by the encoder over length-$P$ windows and pairwise interaction terms; in the worst case this scales as $O(N^2)$, while our implementation uses a sparse $K$-NN adjacency ($K{=}5$) to restrict eligible pairs. Ridge and the MLP are lightweight.
At inference, CRC adds only a small post-hoc module on top of a frozen backbone forecast: a single pass through the encoder/MLP and a frozen per-unit deployment lookup (no iterative optimization at test time).
A detailed complexity discussion is deferred to Appendix~\ref{sec:app_complexity}.

\section{Experiments}
\label{sec:experiments}

\subsection{Experimental Setup}

\paragraph{Datasets and Backbones.}
We evaluate CRC on seven standard long-sequence benchmarks~\cite{zhou2021informer,wu2021autoformer} at horizons $\{96,192,336,720\}$.
\textbf{ETTh1/ETTh2} and \textbf{ETTm1/ETTm2} include electricity load and oil temperature at hourly and 15-minute resolutions, \textbf{Weather} has 21 meteorological channels, \textbf{Electricity} and \textbf{Traffic} are large collections of power and road-occupancy series.
Together they span relatively smooth dynamics to highly non-stationary behavior.

CRC is applied as a plug-in post-hoc corrector to four forecasting backbones:
TimeXer, TimesNet, PatchTST, and DLinear~\cite{wang2024timexer,wu2022timesnet,nie2022time,zeng2023transformers}.
Unless otherwise noted, we report horizon-averaged MSE (our main metric) and MAE averaged over three random seeds.
We additionally report NDR as a diagnostic signal of harmful updates and use it to interpret when selective deployment stabilizes test-time MSE gains.

\paragraph{Benchmark scope.}
Our primary head-to-head comparison is against PIR~\cite{liu2025improving} under the standard long-horizon splits with identical backbone checkpoints for both methods (Section~\ref{sec:comparison_pir}).
Adapter- and calibration-style systems often use different objectives or protocols; we relate them qualitatively in \Cref{sec:related}.

\paragraph{Evaluation Protocol.}
Our primary goal is to improve forecasting accuracy with \textbf{MSE as the main metric}.
We report horizon-averaged MSE (with MAE as a secondary metric) and additionally use \textbf{Non-Degradation Rate (NDR)} as a stability diagnostic:
the fraction of $(\text{node}, \text{horizon})$ pairs for which the corrected prediction is no worse than the baseline (ties count as non-degrading).
Unless stated otherwise, we report test-set NDR.

\paragraph{Training protocol.}
Implementation details are in Appendix~\ref{app:training}: two-stage training (encoder pre-training, then hybrid corrector), AdamW/Adam, correlation-based $K$-NN adjacency ($K{=}5$), $Q{=}0.80$, $\varepsilon{=}0.01$.
PIR and CRC use the same standard splits and frozen backbone checkpoints; post-hoc modules are tuned only on training/validation folds.

\paragraph{Deployment calibration.}
At each $(\text{variable}, \text{horizon})$ unit, CRC selects on the validation split among three frozen candidates (i.e., the backbone forecast, ridge-only correction, and hybrid correction) by lowest validation loss, then fixes that path for test-time deployment without access to test labels.
This matches the validation-calibrated rule in Section~\ref{sec:deployment_rule} (Eq.~\eqref{eq:final_new}). Quantile clipping and shrink-to-base use the defaults above.
Full instantiation is in Appendix~\ref{app:rule_details}.

\subsection{Main Results: Robust MSE Improvements via Selective Deployment}
\label{sec:main_results}

We report (i) horizon-averaged accuracy over seven benchmarks and four backbones (\Cref{tab:summary}), (ii) a deployment ablation with test-time NDR on Electricity--DLinear at $H{=}192$ (\Cref{tab:ablation}), and (iii) a controlled G0--G5 ladder on three representative single-horizon settings that isolates always-deploy vs.\ selective policies (\Cref{tab:table_r1}, seed 2021, not horizon-averaged).
Appendix~\ref{appendix:camera_ready} complements (iii) with seed variability on B1. PIR comparison and calibration sensitivities follow in later subsections.

\input{table/summary}

Across our benchmarks, CRC improves MSE/MAE in most settings. We first summarize headline accuracy, then deployment mechanisms.

\subsubsection{Accuracy: MSE/MAE Improvements}

\noindent
\begin{minipage}[t]{0.51\textwidth}
  \vspace{0pt}%
  \Cref{tab:summary} reports forecasting accuracy with CRC relative to each backbone.
  Across the benchmark suite, CRC improves MSE and MAE in most dataset--backbone pairs; in \Cref{tab:summary} (7 datasets $\times$ 4 backbones), CRC improves horizon-averaged MSE in 26 out of 28 dataset--backbone pairs (the two exceptions are TimeXer on ETTh2 and ETTm1).
  \Cref{tab:summary} also contains a minority of flat or slightly negative entries, which we interpret as weak residual evidence or backbone-specific interactions rather than systematic failure of selective gating.
  Validation-calibrated selective deployment allows strong gains wherever structured residuals dominate, while reducing the frequency of harmful updates (as reflected by NDR where we evaluate it, \Cref{tab:ablation}).

  \vspace{2pt}
  The gains are particularly pronounced on structurally complex benchmarks such as
  \textbf{Traffic}, \textbf{Electricity}, and \textbf{Weather},
  where residual patterns are more likely to contain exploitable structure.
  In these settings, CRC benefits from combining conservative recovery with conditional refinement,
  producing stronger improvements while maintaining a stable correction behavior.
\end{minipage}\hfill
\begin{minipage}[t]{0.465\textwidth}
  \vspace{0pt}%
  \centering
  \setlength{\abovecaptionskip}{3pt}
  \setlength{\belowcaptionskip}{0pt}
  \captionof{table}{%
    Ablation of CRC components on \textbf{Electricity} (\textbf{DLinear} backbone, horizon $H{=}192$).
    \textbf{MSE} is reported at this horizon (lower is better). \textbf{NDR} is higher (better).%
  }
  \label{tab:ablation}
  \vspace{1pt}
  \input{table/ablation_body}
\end{minipage}

\vspace{0.35\baselineskip}
\noindent
By contrast, when residual evidence is weak (e.g., \textbf{ETTm1}), CRC tends to abstain and yields smaller gains rather than forcing updates that can hurt MSE.

\subsubsection{Controlled Comparison: Always-on vs.\ Selective Deployment}

\input{table/table_r1}

To isolate the effect of deployment decisions, we compare the frozen baseline (G0), an always-deploy corrector under the controlled protocol (G1), and four gated or ablated variants (G2--G5) in \Cref{tab:table_r1}.
On B1 (ETTh1, DLinear, $H{=}96$), G1 increases MSE from $0.3962$ to $0.4665$ with NDR $8.2\%$, whereas the ridge-only validation gate (G2) reaches NDR $97.9\%$ but only modest MSE improvement ($0.3941$).
Full CRC (G5) reduces MSE to $0.3780$ while maintaining NDR $80.3\%$. Similar trends hold on B2 and B3.
These results show that conservative gating alone is insufficient for the strongest accuracy gains, while unrestricted deployment of corrections can be harmful, combining selective deployment with conditional refinement yields a more favorable accuracy--stability tradeoff.

\subsubsection{Diagnostics: Harmful-Update Frequency (NDR)}

\Cref{tab:ablation} shows that the full CRC model achieves a high NDR of $95.0\%$ on the DLinear--Electricity setting,
indicating that the proposed deployment rule avoids degradation in the large majority of variable--horizon cases.
Removing any component of the deployment rule reduces NDR substantially, indicating that stability arises from the overall selective deployment policy rather than a single heuristic.

\subsubsection{Takeaway}

The three tables highlight different aspects of CRC: overall forecasting accuracy (\Cref{tab:summary}), the contribution of individual components (\Cref{tab:ablation}), and the cost of deploying corrections indiscriminately (\Cref{tab:table_r1}).
Together they indicate that gains depend on \emph{whether} refinements are released (the deployment stack), not only on how expressive the residual head is.
Appendix~\ref{appendix:camera_ready} reports mean~$\pm$~std seed variability on the same representative B1 setting (\Cref{tab:multiseed_b1}).

\subsubsection{Operational tradeoffs and calibration robustness}
\label{sec:operational_knobs}

Appendix~\ref{appendix:extra_analyses} studies two deployment-time sensitivities: inference-time $\tau$ rescaling, which traces an accuracy--stability tradeoff (test MSE vs.\ test NDR) via clipping aggressiveness, and the validation-ratio budget used to fit the frozen rule (\(5\%\)--\(20\%\)), where headline metrics remain relatively stable.

\subsection{Comparison with PIR: Effectiveness of Validation-Calibrated Deployment}
\label{sec:comparison_pir}

We next compare CRC with the recent post-forecasting identification and revision (PIR)
framework~\cite{liu2025improving},
a strong post-hoc correction baseline that performs instance-level retrieval and revision,
but does not explicitly calibrate correction through a validation-calibrated deployment rule designed to stabilize test-time error.

\input{table/pir_comparison_avg}

\Cref{tab:pir_comparison_avg} reports relative error differences $\Delta = \text{PIR} - \text{CRC}$ (positive favors CRC). CRC is lower on 18 of 28 dataset--backbone pairs (PIR lower on 9, one tie).
The two methods address different failure modes: PIR selects which instances to revise, while CRC gates whether a proposed correction is deployed, so the comparison is not meant as a universal ranking.
Appendix~\ref{sec:app_ndr_pir} further reports horizon-wise NDR on Weather, where CRC often yields stronger stability diagnostics when harmful-update suppression is prioritized.

\begin{figure*}[!t]
  \centering
  \includegraphics[width=\textwidth]{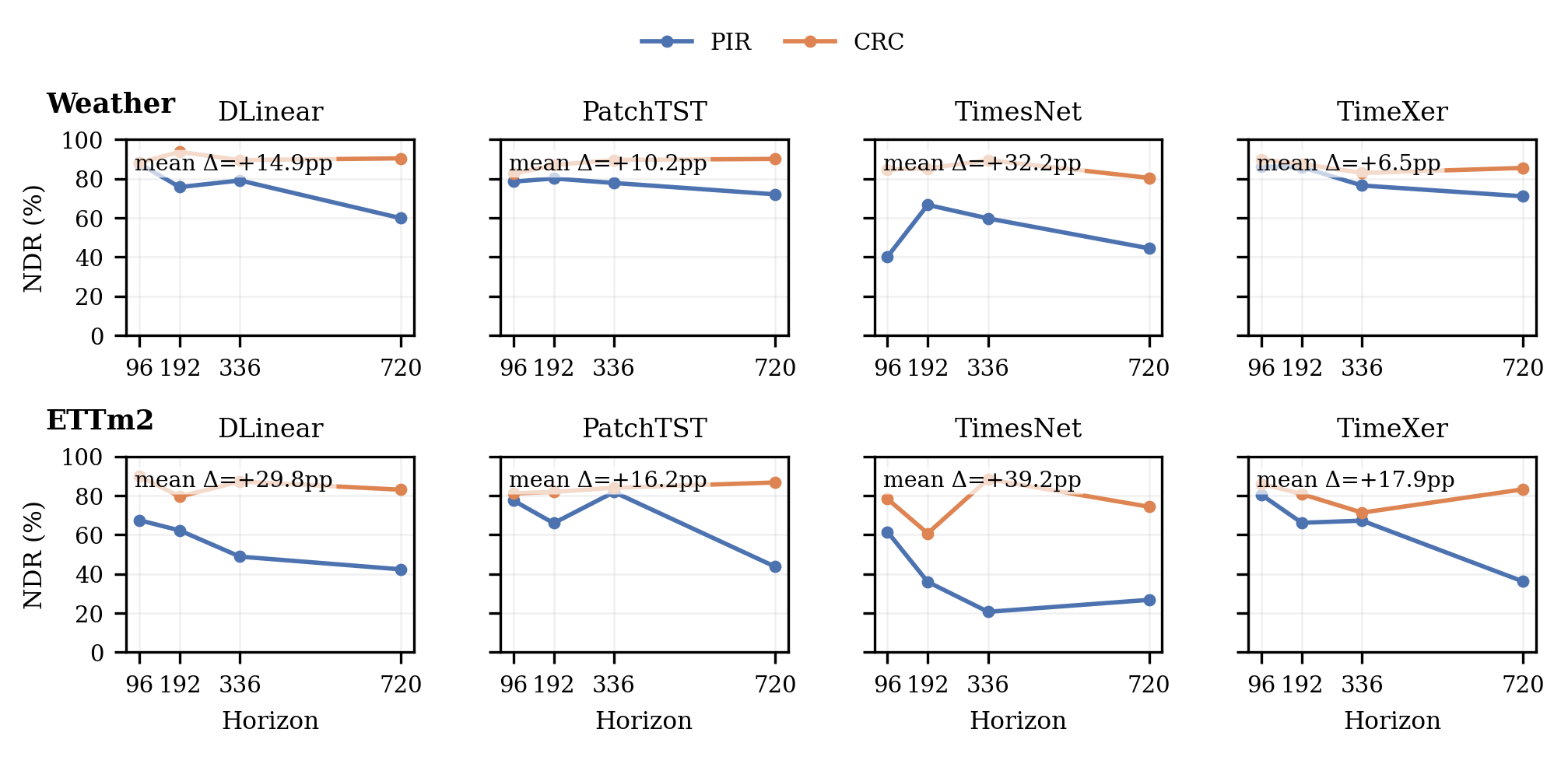}
  \caption{\textbf{NDR stability curves.} Non-degradation rate (NDR, \%) of PIR vs.\ CRC across horizons on \textbf{Weather} and \textbf{ETTm2} (higher is better). Each panel corresponds to one backbone; annotations report mean $\Delta$NDR in percentage points (pp). Exact values are provided in Appendix Table~\ref{tab:ndr_weather_pir_crc}.}
  \label{fig:ndr_curves}
\end{figure*}

As shown in \Cref{fig:ndr_curves}, CRC yields consistently higher NDR than PIR across horizons on both datasets for most backbones, indicating fewer harmful updates under selective deployment.
Notably, the mean $\Delta$NDR improvements are largest on TimesNet, aligning with CRC's goal of suppressing correction-induced degradation.

\subsection{Ablation Study: What Makes Selective Correction Work?}
\label{sec:ablation}

We zoom in on the DLinear--Electricity panel in \Cref{tab:ablation}.
Firewall-style ablations appear as Group~2 in the same table; this subsection focuses on encoder and hybrid-head design (Groups~1 and~3).
We report component-wise ablations on one backbone--dataset pair to control computational cost; \Cref{tab:summary} already varies datasets and backbones for end-to-end accuracy.

\paragraph{Structured representation matters.}
Removing the direction-aware encoder or replacing it with simpler features (Group~3)
consistently weakens performance relative to the full model.
This indicates that selective correction benefits from representations that expose structured residual patterns.

\paragraph{Conservative recovery and conditional refinement play complementary roles.}
Using an MLP-only corrector (Group~1a) leads to severe degradation,
with NDR dropping sharply to 70.0\%,
showing that expressive refinement alone is too unstable.
By contrast, a ridge-only corrector remains much safer but underpowered.
These results are consistent with the conservative--conditional decomposition: the ridge term is relatively stable, while nonlinear refinement helps most when paired with deployment constraints.

Overall, the ablation aligns with the design of CRC: structured residual evidence enables correction, conservative recovery stabilizes it, and the hybrid head is most effective under selective deployment.
Encoder dynamics are illustrated in Appendix~\ref{appendix:extra_analyses} (\cref{fig:dynamic_influence_snapshots}); inference-time $\tau$ scaling and validation-ratio robustness are summarized in Section~\ref{sec:operational_knobs}.

\section{Discussion}
\label{sec:discussion}

\paragraph{Selective deployment for robust MSE gains.}
CRC focuses on gating when to apply post-hoc updates to a frozen forecaster.
When residual evidence is weak, deploying a correction can increase test error even if the corrector fits the training data well.
NDR is therefore used as an operational diagnostic of harmful-update frequency, rather than as a formal safety certificate.

\paragraph{Operational guidelines and monitoring signals.}
CRC exposes deployment-time diagnostics (e.g., clipping/abstention rates and selected-path frequencies) and simple knobs that trade off accuracy and stability (\Cref{sec:operational_knobs}).
A fail-safe path always remains: revert to the frozen backbone (or ridge floor only).

\paragraph{Relation to UQ/abstention and limitations.}
Conformal and selective prediction~\cite{shafer2008tutorial,geifman2019selectivenet} focus on coverage or deferral. CRC instead gates post-hoc point refinements to avoid local degradations, and can be combined with interval methods when needed.
As a hold-out-calibrated procedure, CRC relies on the validation period being representative of deployment conditions; under substantial validation--test shift (or backbone changes), the frozen rule may require recalibration.
On B2 in \Cref{tab:table_r1}, the ridge-only gate keeps NDR high but MAE slightly above the frozen baseline, which is a concrete case where conservative gating is not free.
Recalibration leaves the backbone fixed and refits only the post-hoc stage on a new validation slice.

\section{Conclusion}
\label{sec:conclusion}

We presented CRC, a plug-and-play framework for \emph{robust residual correction via selective deployment} in multivariate forecasting.
It combines a direction-aware encoder with a hybrid residual corrector and a validation-calibrated deployment firewall that stabilizes post-hoc refinement and improves test-time MSE in most evaluated settings.
The main conclusion is \emph{post-hoc correction should be selective, not universal}: robust MSE gains require controlling \emph{when} updates deploy rather than applying residual regression indiscriminately. CRC operationalizes this principle through validation-calibrated gates.

Limitations include reliance on a fixed graph prior (we do not perform causal discovery), conservative gating that may suppress rare large improvements, and dependence of the frozen rule on the validation period, including under distribution shift between validation and test (while robustness to validation budget within the standard protocol is summarized in Section~\ref{sec:operational_knobs}).
Future work includes learned or adaptive graphs, stronger and pretrained backbones, and benchmarks with pronounced spikes or zero inflation to stress-test robustness mechanisms.

\begin{ack}
This work is supported by HKUST-GZ Startup Fund G0101000338.

The authors declare no competing interests.
\end{ack}

\bibliographystyle{plain}
\bibliography{crc_neurips}

\clearpage
\appendix
\section{Theoretical Analysis and Proofs}
\label{appendix:theory}

This appendix collects design-oriented statements under explicit assumptions.
They motivate why clipping and direction-consistent gating can limit harmful updates; they are not test-time certificates under distribution shift or arbitrary dependence across units.
Proposition~\ref{prop:val-safe} holds by construction when validation selects among candidates that include the baseline.
Remark~\ref{remark:pand} gives informal PAND motivation only and is \emph{not} a valid finite-sample lower bound on test NDR.
The main paper's primary empirical target remains reduced test MSE.

We first give pointwise and validation-level non-degradation statements under stated conditions,
then informal probably approximately non-degrading (PAND) motivation (Remark~\ref{remark:pand}).
Finally, we discuss when a direction-aware encoder can aid recovery of systematic residual structure.

\subsection{Validation-Calibrated Selection Properties}

Let $Y$ denote the ground-truth future and $\widehat{Y}^{\mathrm{base}}$ the baseline forecast.
Let $\Delta_{\mathrm{lin}}$ denote the linear (ridge) correction.
We define the residual after the linear floor as
\[
e \;=\; Y - \big(\widehat{Y}^{\mathrm{base}} + \Delta_{\mathrm{lin}}\big).
\]
Let $\delta$ denote the nonlinear correction proposal, and let the clipped nonlinear update
be $\delta'=\mathrm{clip}(\delta,-\tau,\tau)$ with threshold $\tau>0$.

\paragraph{Oracle condition vs.\ deployable rule.}
At test time, the true residual $e$ is unobservable. Therefore, the condition below should be
understood as an \emph{oracle sufficient condition} for pointwise non-degradation.
CRC does \emph{not} assume test-time access to $e$; instead, it enforces a conservative
\emph{training/validation-time gating} using observable residuals on held-out data and \emph{freezes} the resulting rule for deployment.
The validation-level selection and shrink-to-base mechanisms (Prop.~\ref{prop:val-safe}) provide
a deployable validation-calibrated selection property (the baseline remains an explicit candidate), without requiring test-time residual access.

\begin{assumption}[Directional gating and bounded update (oracle sufficient condition)]
\label{ass:gate-clip}
(Oracle condition.) The clipped nonlinear update is direction-aligned with
the post-linear residual and bounded in magnitude, i.e.,
\[
\mathrm{sign}(\delta')=\mathrm{sign}(e),
\qquad
0 \le |\delta'| \le \min\{|e|,\tau\}.
\]
\end{assumption}

\begin{proposition}[Pointwise non-degradation under MAE/MSE]
\label{prop:pointwise}
Under Assumption~\ref{ass:gate-clip}, the pointwise error does not increase after applying
the clipped nonlinear update:
\[
\big|e-\delta'\big| \;\le\; |e|.
\]
Moreover, the same condition implies pointwise non-increase of MSE:
\[
(e-\delta')^2 \;\le\; e^2.
\]
\end{proposition}

\begin{proof}
Under Assumption~\ref{ass:gate-clip}, $\mathrm{sign}(\delta')=\mathrm{sign}(e)$ and $|\delta'|\le |e|$.
Hence $|e-\delta'|=||e|-|\delta'||\le |e|$, proving the MAE claim.
Squaring both sides yields $(e-\delta')^2 \le e^2$, proving the MSE claim.
\end{proof}

\begin{proposition}[Validation-level non-degradation w.r.t.\ the baseline]
\label{prop:val-safe}
On the validation set, for each evaluation unit (e.g., each $(i,h)$ pair under the NDR definition),
select the best-performing predictor among
\[
\widehat{Y}^{\mathrm{base}}, \qquad
\widehat{Y}^{\mathrm{base}} + \Delta_{\mathrm{lin}}, \qquad
\widehat{Y}^{\mathrm{base}} + \Delta_{\mathrm{lin}} + \delta',
\]
according to the validation loss, and freeze this choice for deployment.
Then the validation error satisfies
\[
\mathcal{L}_{\mathrm{val}}(\mathrm{CRC}) \;\le\; \mathcal{L}_{\mathrm{val}}(\mathrm{base}),
\]
and consequently $\mathrm{NDR}_{\mathrm{val}} = 100\%$ at the same evaluation granularity.
If a positive shrink-to-base margin $\varepsilon>0$ is used, some units may default to the baseline
when the best CRC path beats the baseline by less than $\varepsilon$ on validation; the baseline
remains an explicit candidate, so per-unit validation loss is still no worse than that of the baseline.
\end{proposition}

\begin{proof}
For each evaluation unit, the selection rule explicitly chooses the candidate with the smallest
validation loss from a set that includes the baseline itself.
Therefore, the deployed predictor cannot incur higher validation loss than the baseline on any unit.
Aggregating over all units yields the stated inequality.
The NDR result follows directly from its definition.
\end{proof}

\begin{remark}[Informal PAND motivation; not a test-set guarantee]
\label{remark:pand}
Let $m$ denote the number of validation evaluation units and define
$Z_k=\mathbb{1}\!\left\{ \ell(\widehat{Y}^{\mathrm{CRC}},Y)
\le \ell(\widehat{Y}^{\mathrm{base}},Y) \right\}$
at unit $k$, where $\ell(\cdot,\cdot)$ is the per-unit loss used in NDR (e.g., MAE or MSE), with
$\mathrm{NDR}_{\mathrm{val}}=\frac{1}{m}\sum_{k=1}^m Z_k$.
High $\mathrm{NDR}_{\mathrm{val}}$ is \emph{encouraging} only under idealized conditions: the deployment rule is treated as fixed when judging test units, and units behave like i.i.d.\ draws from a stable per-unit success law.
In practice, the rule is fit and selected using validation data (selection bias), units across $(i,h)$ may be dependent, and a rigorous validation-to-test link requires separate concentration for validation \emph{and} test empirical means---not a one-step corollary of Hoeffding~\cite{hoeffding1963probability} on validation alone.
We therefore do \emph{not} state a finite-sample inequality of the form
$\mathrm{NDR}_{\mathrm{test}}\ge \mathrm{NDR}_{\mathrm{val}}-\sqrt{\log(2/\alpha)/(2m)}$;
PAND language is qualitative motivation to read alongside reported test NDR and the sensitivity analyses in Appendix~\ref{appendix:extra_analyses} (Section~\ref{sec:operational_knobs}), not a deployment certificate under shift.
\end{remark}

\subsection{When Does the Direction-Aware Encoder Help Recovery?}

Consider the following data-generating process for node $i$:
\begin{equation}
\label{eq:dgp}
Y_{:,i}
\;=\;
f_i(W_{:,i})
\;+\;
\sum_{j\in \mathrm{Pa}(i)} g_{j\to i}(W_{:,j})
\;+\;
\varepsilon_i,
\end{equation}
where the baseline $\widehat{Y}^{\mathrm{base}}_i$ already captures $f_i(\cdot)$
but may miss the cross-node systematic term.
Define the baseline residual
\[
r_i \;=\; Y_{:,i} - \widehat{Y}^{\mathrm{base}}_i .
\]

Let the encoder output be
\[
Z_i
\;=\;
\alpha_{i\to i} U_i
\;+\;
\sum_{j\in \mathcal{N}(i)} \alpha_{j\to i} U_j ,
\]
where $U_j$ denotes node-wise latent features extracted from historical windows,
and $\mathcal{N}(i)$ covers $\mathrm{Pa}(i)$.

\begin{remark}[Ridge floor as design motivation (not a recovery guarantee)]
\label{prop:recover}
When cross-node effects admit a local linearization and the encoder features $Z_i$ approximately span the associated residual subspace, fitting a ridge map on $Z_i$ provides a ridge-regularized linear approximation to the residual---a conservative correction floor rather than a certified recovery of systematic error.
Embedding distortion, noise, and model misspecification affect this approximation in practice.
The encoder supplies direction-aware inductive bias rather than causal discovery; nonlinear refinement and the validation-calibrated deployment rule guard updates beyond this floor.
\end{remark}

\section{Implementation Details}

\subsection{Encoder, priors, and directional gates (compact)}
\label{app:feature_priors}
\label{app:fpair}

\paragraph{Multi-scale priors.}
We augment the learned encoder output with fixed per-node priors $\mathbf{F}$ computed over multiple temporal scales:
basic distribution statistics (mean/variance/IQR), temporal dynamics (EMA and absolute-difference energy), correlation features (multi-scale lagged auto-correlation), and frequency features (band power via rFFT).
Priors are aggregated through the row-normalized adjacency $\mathbf{A}_{\mathrm{norm}}$ and serve as explicit inductive bias under non-stationarity.

\paragraph{Pairwise interaction module.}
The shared ordered-pair module $f_{\mathrm{pair}}$ consumes windows $(W_{:,:,i},W_{:,:,j})$ and outputs an interaction embedding plus a directional gate.
Internally it produces a $2{\times}2$ matrix $\mathbf{M}_{ij}$; the off-diagonal entry corresponding to $j{\to}i$ is passed through $\tanh$ to form $\alpha_{j\to i}$ (directional evidence, not causal identification).

\subsection{Training Hyperparameters}
\label{app:training}

CRC is implemented in PyTorch~\cite{paszke2019pytorch} and trained on NVIDIA GPUs using a two-stage procedure:
(i) multi-task pre-training of the encoder, then (ii) fitting the conservative--conditional corrector.
We use AdamW~\cite{loshchilov2017decoupled} (encoder) and Adam~\cite{kingma2014adam} (MLP) with learning rate $10^{-4}$, batch size 32, and early stopping on validation MAE.
The graph prior is a correlation-based $K$-NN adjacency built on the training split with $K{=}5$ and then fixed.
Unless stated otherwise, quantile clipping uses $Q{=}0.80$ and shrink-to-base margin uses $\varepsilon{=}0.01$.
All post-hoc modules are tuned on the standard train/val folds using frozen backbone checkpoints.

\subsection{Deployment Rule Instantiation}
\label{app:rule_details}

CRC deploys a frozen, validation-calibrated choice among a small candidate set at the NDR granularity (per $(i,h)$ unit).
For each unit we select the lowest validation-loss predictor from
\(\{\widehat{Y}^{\mathrm{base}},\ \widehat{Y}^{\mathrm{base}}+\Delta_{\mathrm{lin}},\ \widehat{Y}^{\mathrm{base}}+\Delta_{\mathrm{lin}}+\delta'\}\),
and freeze the selection for test.
Equivalently, \(\widehat{Y}=\widehat{Y}^{\mathrm{base}}+w_1\Delta^{\mathrm{base}}+w_2\Delta^{\mathrm{clip}}\) with \((w_1,w_2)\in\{(0,0),(1,0),(1,1)\}\) chosen on validation.
The clipping threshold $\tau$ is estimated from validation via quantile clipping (default $Q{=}0.80$), and shrink-to-base reverts to the baseline unless the best candidate improves validation loss by at least $\varepsilon{=}0.01$.

\subsection{Computational Complexity}
\label{sec:app_complexity}

The ridge term admits a closed-form solution, the MLP is small, and the encoder operates on fixed windows.
Overall complexity is $O(BPN + N^2 d)$ with embedding dimension $d$, batch $B$, window $P$, and $N$ variables; the $N^2 d$ term reflects the pairwise interaction blocks (Section~\ref{sec:encoder}), consistent with the $O(N^2)$ scaling discussed in the main text.
This makes CRC efficient to train and easy to plug into diverse backbones.

\input{appendix/extra_analyses}

\input{appendix/camera_ready_extras}

\section{Additional Experimental Results}
\label{appendix:additional_results}

This section presents additional experimental results that complement the
summary statistics reported in the main text by providing horizon-wise results
underlying \Cref{tab:pir_comparison_avg}. All comparisons use identical
backbone checkpoints, splits, and post-hoc fitting protocols. \Cref{tab:main-results}
reports absolute errors (with/without CRC), while \Cref{tab:pir_comparison_full_one_table}
reports PIR$-\,$CRC differences (positive favors CRC), exposing horizon-specific behavior that
can be obscured by horizon-averaged metrics.

\input{table/main-results}
\input{table/pir_comparison_full_one_table}


\end{document}

%% file: table/summary.tex
\begin{table*}[t]
\centering
\caption{Horizon-averaged MSE and MAE improvements of CRC across datasets and backbones
(averaged over horizons {96, 192, 336, 720}). Improvements are relative to the baseline and positive values indicate gains. Full horizon-wise results are reported in Appendix Table~\ref{tab:main-results}.}
\label{tab:summary}
\resizebox{\textwidth}{!}{
\begin{tabular}{l c c c c c c c c}
\toprule
\textbf{Dataset} & \textbf{Model} & \textbf{MSE (Baseline)} & \textbf{MSE (CRC)} & \textbf{Imp MSE (\%)} & \textbf{MAE (Baseline)} & \textbf{MAE (CRC)} & \textbf{Imp MAE (\%)} \\
\midrule
\multirow{4}{*}{\textbf{ETTh1}}
 & DLinear & 0.4603 & 0.4582 & \textbf{0.46} & 0.4568 & 0.4554 & \textbf{0.29} \\
 & PatchTST & 0.4478 & 0.4468 & \textbf{0.22} & 0.4476 & 0.4454 & \textbf{0.50} \\
 & TimeXer & 0.4441 & 0.439 & \textbf{1.14} & 0.4451 & 0.444 & \textbf{0.24} \\
 & TimesNet & 0.4565 & 0.4562 & \textbf{0.06} & 0.4528 & 0.4525 & \textbf{0.06} \\
\midrule
\multirow{4}{*}{\textbf{ETTh2}}
 & DLinear & 0.5641 & 0.5383 & \textbf{4.57} & 0.5194 & 0.497 & \textbf{4.32} \\
 & PatchTST & 0.3807 & 0.3794 & \textbf{0.35} & 0.4092 & 0.4078 & \textbf{0.36} \\
 & TimeXer & 0.3769 & 0.3774 & -0.13 & 0.4006 & 0.4016 & -0.24 \\
 & TimesNet & 0.4082 & 0.4082 & \textbf{0.01} & 0.4224 & 0.4219 & \textbf{0.10} \\
\midrule
\multirow{4}{*}{\textbf{ETTm1}}
 & DLinear & 0.404 & 0.3935 & \textbf{2.60} & 0.4077 & 0.405 & \textbf{0.65} \\
 & PatchTST & 0.3888 & 0.3863 & \textbf{0.64} & 0.4029 & 0.4019 & \textbf{0.24} \\
 & TimeXer & 0.3817 & 0.3846 & -0.76 & 0.3968 & 0.3982 & -0.35 \\
 & TimesNet & 0.4199 & 0.4142 & \textbf{1.37} & 0.4208 & 0.4195 & \textbf{0.33} \\
\midrule
\multirow{4}{*}{\textbf{ETTm2}}
 & DLinear & 0.3547 & 0.3375 & \textbf{4.86} & 0.4018 & 0.3808 & \textbf{5.21} \\
 & PatchTST & 0.2923 & 0.2888 & \textbf{1.19} & 0.3353 & 0.333 & \textbf{0.68} \\
 & TimeXer & 0.2735 & 0.2688 & \textbf{1.74} & 0.3214 & 0.32 & \textbf{0.44} \\
 & TimesNet & 0.2947 & 0.2918 & \textbf{1.01} & 0.3328 & 0.3326 & \textbf{0.07} \\
\midrule
\multirow{4}{*}{\textbf{Electricity}}
 & DLinear & 0.2178 & 0.1801 & \textbf{17.30} & 0.3082 & 0.2775 & \textbf{9.95} \\
 & PatchTST & 0.2038 & 0.1772 & \textbf{13.05} & 0.2943 & 0.2722 & \textbf{7.49} \\
 & TimeXer & 0.1724 & 0.1634 & \textbf{5.28} & 0.2707 & 0.2609 & \textbf{3.65} \\
 & TimesNet & 0.1967 & 0.1864 & \textbf{5.22} & 0.2963 & 0.2877 & \textbf{2.92} \\
\midrule
\multirow{4}{*}{\textbf{Traffic}}
 & DLinear & 0.6384 & 0.5042 & \textbf{21.03} & 0.3946 & 0.3184 & \textbf{19.31} \\
 & PatchTST & 0.4914 & 0.4646 & \textbf{5.47} & 0.3151 & 0.2973 & \textbf{5.66} \\
 & TimeXer & 0.4705 & 0.4653 & \textbf{1.11} & 0.2875 & 0.2804 & \textbf{2.48} \\
 & TimesNet & 0.6322 & 0.6039 & \textbf{4.47} & 0.3353 & 0.3138 & \textbf{6.40} \\
\midrule
\multirow{4}{*}{\textbf{Weather}}
 & DLinear & 0.2654 & 0.2385 & \textbf{10.13} & 0.317 & 0.2833 & \textbf{10.61} \\
 & PatchTST & 0.2582 & 0.2362 & \textbf{8.52} & 0.2796 & 0.2676 & \textbf{4.33} \\
 & TimeXer & 0.2406 & 0.2318 & \textbf{3.67} & 0.2707 & 0.2651 & \textbf{2.07} \\
 & TimesNet & 0.2631 & 0.2401 & \textbf{8.75} & 0.2888 & 0.2746 & \textbf{4.94} \\
\midrule
\bottomrule
\end{tabular}}
\end{table*}

%% file: table/ablation_body.tex
\begingroup
\small
\sisetup{
  table-number-alignment = center,
  detect-weight = true
}
\setlength{\extrarowheight}{1.4pt}
\setlength{\tabcolsep}{2.4pt}
\renewcommand{\arraystretch}{1.16}
\begin{tabularx}{\linewidth}{@{}>{\raggedright\arraybackslash}X S[table-format=1.3] S[table-format=2.1]@{}}
  \toprule
  \textbf{Variant} & {\textbf{MSE} $\downarrow$} & {\textbf{NDR} $\uparrow$} \\
  \midrule
  DLinear baseline & 0.198 & \multicolumn{1}{c}{---} \\
  \textbf{CRC (full, ours)} & \textbf{0.173} & \textbf{95.0} \\
  \addlinespace[2pt]
  \multicolumn{3}{@{}l}{\textbf{(1) \textit{Hybrid corrector}}} \\
  \hspace{0.4em}1a.\ w/o ridge floor (\mbox{\textsc{mlp}-only}) & 0.181 & 70.0 \\
  \hspace{0.4em}1b.\ w/o \mbox{\textsc{mlp}} $\Delta$ (ridge-only) & 0.192 & 94.0 \\
  \addlinespace[2.5pt]
  \multicolumn{3}{@{}l}{\textbf{(2) \textit{Safety firewall}}} \\
  \hspace{0.4em}2a.\ w/o quantile clipping & 0.176 & 85.0 \\
  \hspace{0.4em}2b.\ w/o direction gating & 0.175 & 88.0 \\
  \hspace{0.4em}2c.\ w/o point-wise selection & 0.174 & 90.0 \\
  \hspace{0.4em}2d.\ w/o shrink-to-base blending & 0.174 & 82.0 \\
  \addlinespace[2.5pt]
  \multicolumn{3}{@{}l}{\textbf{(3) \textit{Direction-aware encoder}}} \\
  \hspace{0.4em}3a.\ w/o graph (self-only) & 0.185 & 93.0 \\
  \hspace{0.4em}3b.\ w/o encoder (simple features) & 0.195 & 90.0 \\
  \bottomrule
\end{tabularx}
\endgroup

%% file: table/table_r1.tex
\begin{table*}[t]
\centering
\small
\setlength{\tabcolsep}{3.5pt}
\caption{Controlled deployment ladder (G0--G5) on B1--B3 (seed 2021).
B1/B2/B3: ETTh1+DLinear at pred\_len 96/336 and Weather+TimesNet at 96.
\textbf{Not} horizon-averaged like \Cref{tab:summary}.
G0 baseline, G1 always-deploy (blend search off; $\tau$/direction gate retained), G2 ridge+validation gate only (no MLP), G3 w/o direction gate, G4 deploy-only encoder, G5 full CRC.
Lower MAE/MSE and higher NDR (\%) are better.}
\label{tab:table_r1}
\resizebox{\textwidth}{!}{%
\begin{tabular}{lccccccccc}
\toprule
Config & B1 MAE & B1 MSE & B1 NDR & B2 MAE & B2 MSE & B2 NDR & B3 MAE & B3 MSE & B3 NDR \\
\midrule
G0 Baseline & 0.4108 & 0.3962 & --- & 0.4654 & 0.4874 & --- & 0.2184 & 0.1685 & --- \\
G1 Always-on & 0.4738 & 0.4665 & 8.2 & 0.5689 & 0.6328 & 10.6 & 0.2326 & 0.1803 & 33.9 \\
G2 Ridge+val gate & 0.4097 & 0.3941 & 97.9 & 0.4693 & 0.4920 & 97.1 & 0.2111 & 0.1626 & 90.0 \\
G3 No dir gate & 0.4079 & 0.3865 & 81.8 & 0.4684 & 0.4786 & 78.5 & 0.2098 & 0.1611 & 81.5 \\
G4 Deploy-only & 0.4097 & 0.3941 & 62.1 & 0.4592 & 0.4838 & 77.8 & 0.2081 & 0.1584 & 79.4 \\
G5 Full CRC & \textbf{0.3914} & \textbf{0.3780} & \textbf{80.3} & \textbf{0.4392} & \textbf{0.4612} & \textbf{79.7} & \textbf{0.1903} & \textbf{0.1408} & \textbf{78.5} \\
\bottomrule
\end{tabular}%
}
\end{table*}

%% file: table/pir_comparison_avg.tex
\begin{table*}[!t]
\centering
\small
\setlength{\tabcolsep}{4pt}
\renewcommand{\arraystretch}{1.10}
\sisetup{
  table-number-alignment = center,
  detect-weight = true,
  detect-inline-weight = math
}

\caption{Horizon-averaged relative forecasting error differences between PIR and CRC
($\Delta=\mathrm{PIR}-\mathrm{CRC}$; positive indicates CRC better).
\textbf{Winner} summarizes the majority sign of $\Delta$ across backbone--horizon pairs
for each dataset.
Full horizon-wise results (including MAE) are reported in
Appendix Table~\ref{tab:pir_comparison_full_one_table}.}

\label{tab:pir_comparison_avg}

\begin{tabular*}{\linewidth}{@{\extracolsep{\fill}} l
  S[table-format=-1.4]
  S[table-format=-1.4]
  S[table-format=-1.4]
  S[table-format=-1.4]
  c}
\toprule
\textbf{Dataset} & {\textbf{TimeXer}} & {\textbf{TimesNet}} & {\textbf{PatchTST}} & {\textbf{DLinear}} & \textbf{Winner} \\
\midrule
ETTh1       & -0.0223 &  0.0000 & -0.0036 & -0.0217 & \textsc{pir} \\
ETTh2       & -0.0076 &  0.0053 &  0.0141 & -0.1673 & \textsc{crc} \\
ETTm1       &  0.0116 &  0.0116 &  0.0042 & -0.0065 & \textsc{crc} \\
ETTm2       &  0.0205 &  0.0138 &  0.0102 & -0.0537 & \textsc{crc} \\
Weather     &  0.0117 &  0.0192 &  0.0371 &  0.0267 & \textsc{crc} \\
Electricity &  0.0066 & -0.0001 &  0.0123 &  0.0154 & \textsc{crc} \\
Traffic     &  0.0124 & -0.0519 &  0.0142 &  0.0226 & \textsc{crc} \\
\bottomrule
\end{tabular*}

\footnotesize\textit{Across 28 dataset--backbone pairs, CRC has lower averaged error in 18 cases (PIR: 9; tie: 1).}
\end{table*}

%% file: appendix/extra_analyses.tex

\section{Supplementary Experimental Analyses}
\label{appendix:extra_analyses}

\subsection{Detailed test-time NDR vs.\ PIR (Weather)}
\label{sec:app_ndr_pir}

Table~\ref{tab:ndr_weather_pir_crc} provides a horizon-wise \textbf{NDR} comparison supporting the main-text PIR vs.\ CRC results.

\input{table/ndr_weather_pir_crc}

\subsection{Temporal evolution of learned directional influences}
\label{sec:app_interpretability}

\begin{center}
    \includegraphics[width=\linewidth, trim=32mm 0mm 28mm 0mm, clip]{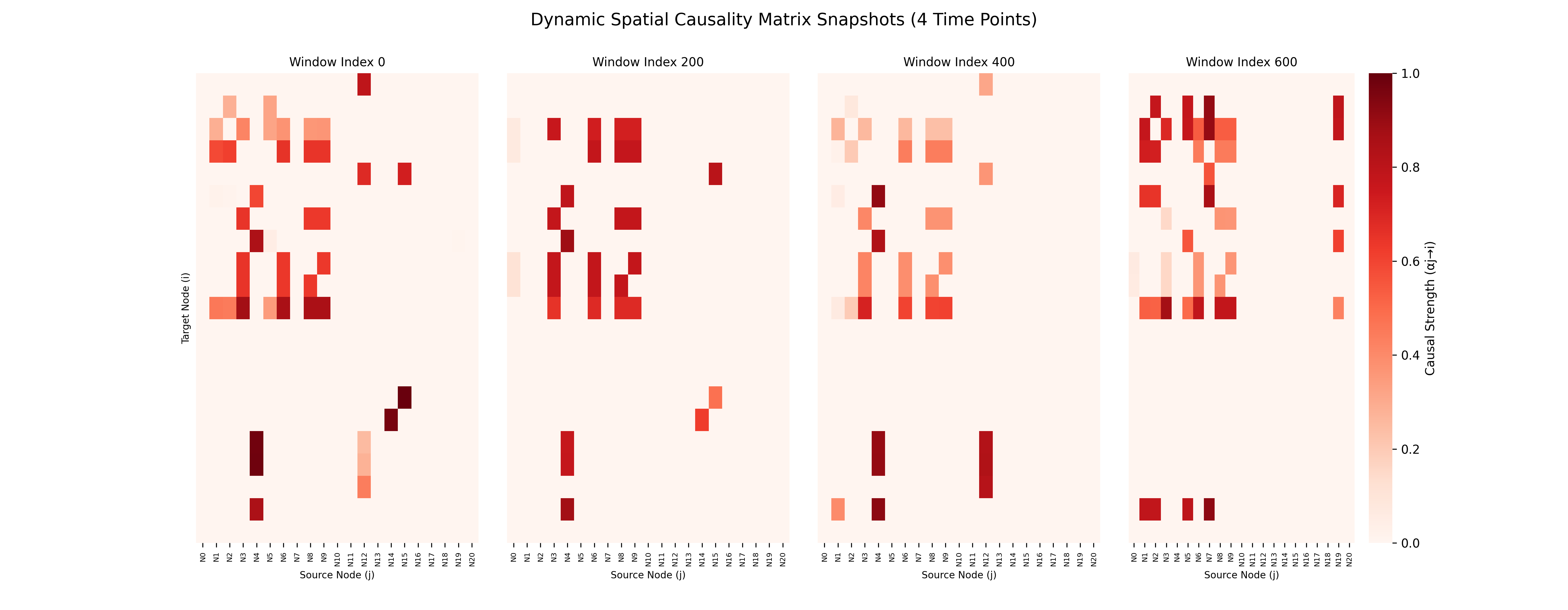}
    \captionof{figure}{\textbf{Temporal evolution of learned directional interaction scores ($N{=}4$).}
    Instantaneous directional scores $S_t[i,j]$ (messaging $j \!\to\! i$) with a \textbf{TimesNet} backbone on \textbf{Weather}.
    Four windows ($\Delta t = 200$) illustrate non-stationary dependencies used as residual evidence; these scores provide directional evidence and are \emph{not} interpreted as identified causal effects.}
    \label{fig:dynamic_influence_snapshots}
\end{center}

\subsection{Sensitivity to clipping and validation calibration}
\label{sec:app_sensitivity}

We examine robustness of the MSE--NDR tradeoff to (i) rescaling the validation-calibrated clipping threshold~$\tau$ and (ii) perturbing validation-side calibration of the frozen intervention rule (main-text discussion: Section~\ref{sec:operational_knobs}).
Figures~\ref{fig:tau_scale_sweep} and~\ref{fig:validation_sensitivity} report the resulting accuracy--reliability tradeoffs and their sensitivity to the amount of validation data used for calibration, respectively.

\begin{center}
  \includegraphics[width=\linewidth]{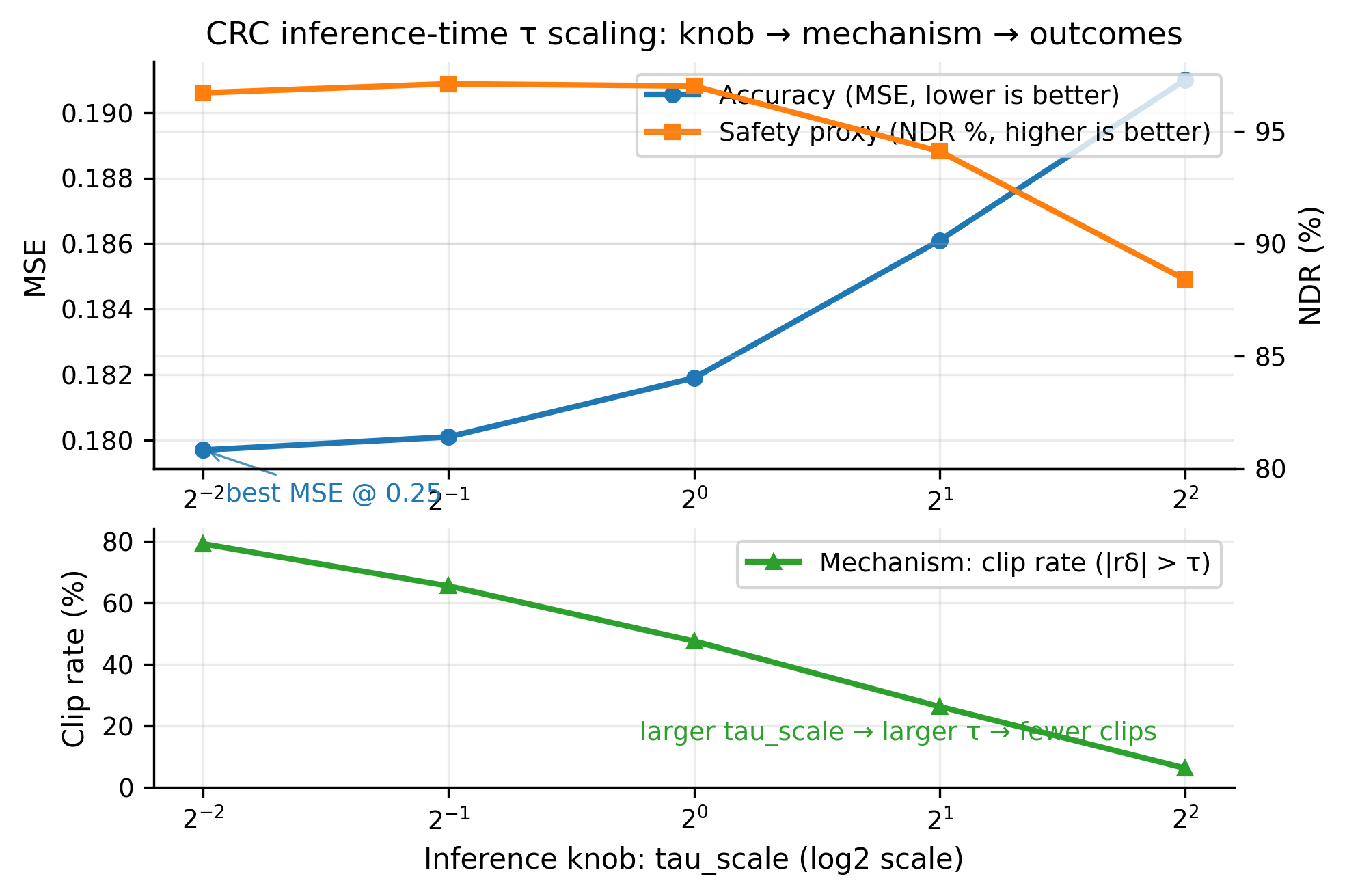}
  \captionof{figure}{MSE--NDR tradeoff under rescaling of the calibrated quantile-clipping threshold~$\tau$.}
  \label{fig:tau_scale_sweep}
\end{center}

\begin{center}
  \includegraphics[width=\linewidth]{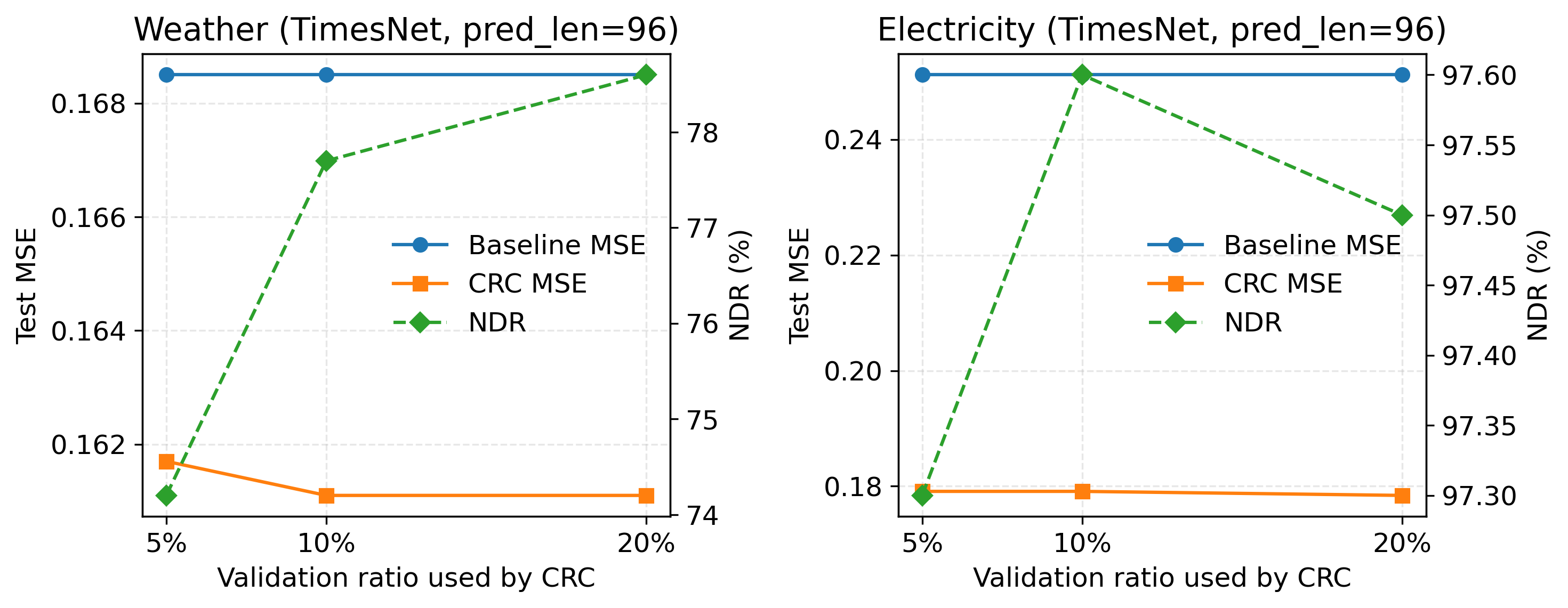}
  \captionof{figure}{Sensitivity of the MSE--NDR profile to validation calibration of the frozen intervention rule.}
  \label{fig:validation_sensitivity}
\end{center}

%% file: table/ndr_weather_pir_crc.tex
\begin{center}
\captionof{table}{Non-Degradation Rate (NDR, \%) comparison between PIR and CRC on \textbf{Weather} (left) and \textbf{ETTm2} (right) across backbones and forecasting horizons. Higher is better.}
\label{tab:ndr_weather_pir_crc}
\small
\setlength{\tabcolsep}{4pt}
\begin{tabular}{@{}c@{\hspace{12pt}}c@{}}
\begin{subtable}[t]{0.48\textwidth}
  \centering
  \caption*{\textbf{Weather}}
  \begin{tabular}{llcc}
    \toprule
    \textbf{Backbone} & \textbf{Horizon} & \textbf{PIR} & \textbf{CRC} \\
    \midrule
    \multirow{4}{*}{DLinear}
      & 96  & 87.7 & \textbf{88.2} \\
      & 192 & 75.7 & \textbf{93.7} \\
      & 336 & 79.1 & \textbf{89.5} \\
      & 720 & 59.8 & \textbf{90.4} \\
    \midrule
    \multirow{4}{*}{PatchTST}
      & 96  & 78.5 & \textbf{82.5} \\
      & 192 & 80.1 & \textbf{87.0} \\
      & 336 & 77.8 & \textbf{89.6} \\
      & 720 & 72.0 & \textbf{90.1} \\
    \midrule
    \multirow{4}{*}{TimesNet}
      & 96  & 40.1 & \textbf{84.7} \\
      & 192 & 66.6 & \textbf{85.2} \\
      & 336 & 59.7 & \textbf{89.5} \\
      & 720 & 44.4 & \textbf{80.4} \\
    \midrule
    \multirow{4}{*}{TimeXer}
      & 96  & 86.3 & \textbf{89.7} \\
      & 192 & 85.9 & \textbf{87.6} \\
      & 336 & 76.5 & \textbf{82.9} \\
      & 720 & 71.0 & \textbf{85.5} \\
    \bottomrule
  \end{tabular}
\end{subtable}
&
\begin{subtable}[t]{0.48\textwidth}
  \centering
  \caption*{\textbf{ETTm2}}
  \begin{tabular}{llcc}
    \toprule
    \textbf{Backbone} & \textbf{Horizon} & \textbf{PIR} & \textbf{CRC} \\
    \midrule
    \multirow{4}{*}{DLinear}
      & 96  & 67.4 & \textbf{90.0} \\
      & 192 & 62.2 & \textbf{79.5} \\
      & 336 & 48.8 & \textbf{87.2} \\
      & 720 & 42.3 & \textbf{83.0} \\
    \midrule
    \multirow{4}{*}{PatchTST}
      & 96  & 77.5 & \textbf{81.1} \\
      & 192 & 65.9 & \textbf{81.9} \\
      & 336 & 81.8 & \textbf{84.0} \\
      & 720 & 43.7 & \textbf{86.7} \\
    \midrule
    \multirow{4}{*}{TimesNet}
      & 96  & 61.5 & \textbf{78.3} \\
      & 192 & 35.9 & \textbf{60.7} \\
      & 336 & 20.6 & \textbf{88.4} \\
      & 720 & 26.7 & \textbf{74.3} \\
    \midrule
    \multirow{4}{*}{TimeXer}
      & 96  & 80.4 & \textbf{86.0} \\
      & 192 & 66.1 & \textbf{80.8} \\
      & 336 & 67.2 & \textbf{71.3} \\
      & 720 & 36.1 & \textbf{83.2} \\
    \bottomrule
  \end{tabular}
\end{subtable}
\end{tabular}
\end{center}

%% file: appendix/camera_ready_extras.tex
\section{Controlled Deployment Ladder and Seed Variability}
\label{appendix:camera_ready}

\subsection{Notes on the Controlled Deployment Ladder}
\label{sec:app_table_r1}

The full G0--G5 controlled deployment ladder is reported in \Cref{tab:table_r1} (Section~\ref{sec:main_results}).
G1 (always-deploy under the protocol) drops NDR sharply relative to G5; G2 can stay safe yet under-improve (B2 MAE above baseline despite high NDR).
Main-table baseline/CRC rows match G0/G5 only as policies: those cells are horizon-averaged and need not equal the single-horizon numbers in \Cref{tab:table_r1}.

\subsection{Representative seed-variability analysis (B1 only)}
\label{sec:app_multiseed}
\input{table/multiseed_b1}

%% file: table/multiseed_b1.tex
\begin{center}
\small
\captionof{table}{Representative seed-variability analysis (B1 only).
ETTh1, DLinear, pred\_len$=$96 (B1).
Mean~$\pm$~std over training seeds $\{2021,2022,2023\}$ for the post-hoc CRC module (frozen backbone).
This table reports an independent three-seed robustness evaluation on the representative B1 setting, complementing the headline main-table results (Table~\ref{tab:summary}).
\Cref{tab:table_r1} is an exception (controlled deployment ladder, seed 2021 only): G5 B1 MAE $0.3914$ vs.\ seed-averaged $0.3991$---both are valid under different reporting protocols.}
\label{tab:multiseed_b1}
\begin{tabular}{lccc}
\toprule
Config & MAE & MSE & NDR (\%) \\
\midrule
G5 Full CRC & $0.3991\pm0.0011$ & $0.3785\pm0.0007$ & $79.3\pm6.3$ \\
G1 Always-on & $0.4828\pm0.0131$ & $0.4936\pm0.0378$ & $7.2\pm1.5$ \\
\bottomrule
\end{tabular}
\end{center}

%% file: table/main-results.tex

\noindent\begin{minipage}{\textwidth}
\centering
\captionof{table}{Results on ETTh1, ETTh2, ETTm1, ETTm2, Weather, Electricity, and Traffic datasets (horizons 96/192/336/720). Lower is better. Bold indicates that CRC outperforms the corresponding baseline for that metric.}
\label{tab:main-results}
\setlength{\tabcolsep}{3pt} 
\resizebox{\textwidth}{!}{
\begin{tabular}{cc*{20}{c}} 
\toprule
\multirow{2}{*}{\textbf{Dataset}} & \multirow{2}{*}{\textbf{Horizon}} &
\multicolumn{2}{c}{\textbf{TimeXer}} & \multicolumn{2}{c}{\textbf{+ CRC}} &
\multicolumn{2}{c}{\textbf{TimesNet}} & \multicolumn{2}{c}{\textbf{+ CRC}} &
\multicolumn{2}{c}{\textbf{PatchTST}} & \multicolumn{2}{c}{\textbf{+ CRC}} &
\multicolumn{2}{c}{\textbf{DLinear}} & \multicolumn{2}{c}{\textbf{+ CRC}} \\
& & \textbf{MSE} & \textbf{MAE} & \textbf{MSE} & \textbf{MAE} &
\textbf{MSE} & \textbf{MAE} & \textbf{MSE} & \textbf{MAE} &
\textbf{MSE} & \textbf{MAE} & \textbf{MSE} & \textbf{MAE} &
\textbf{MSE} & \textbf{MAE} & \textbf{MSE} & \textbf{MAE} \\
\midrule
\multirow{4}{*}{\textbf{ETTh1}}
& 96  & 0.3818 & 0.4029 & \textbf{0.3790} & \textbf{0.3989} 
& 0.3984 & 0.4163 & 0.3986 & \textbf{0.4152} 
& 0.3792 & 0.3988 & \textbf{0.3790} & \textbf{0.3898}
& 0.3962 & 0.4108 & \textbf{0.3873} & \textbf{0.4099} \\ 
& 192 & 0.4285 & 0.4355 & \textbf{0.4100} & \textbf{0.4341}
& 0.4373 & 0.4399 & 0.4473 & \textbf{0.4374}
& 0.4240 & 0.4304 & \textbf{0.4218} & \textbf{0.4252}
& 0.4450 & 0.4404 & \textbf{0.4392} & \textbf{0.4357} \\ 
& 336 & 0.4672 & 0.4494 & 0.4704 & 0.4500 
& 0.4722 & 0.4584 & \textbf{0.4650} & 0.4586 
& 0.4686 & 0.4578 & \textbf{0.4663} & 0.4651 
& 0.4874 & 0.4654 & \textbf{0.4820} & 0.4655 \\ 
& 720 & 0.4988 & 0.4926 & \textbf{0.4967} & 0.4931 
& 0.5181 & 0.4965 & \textbf{0.5140} & 0.4989 
& 0.5193 & 0.5036 & 0.5200 & \textbf{0.5015} 
& 0.5126 & 0.5104 & 0.5242 & 0.5106 \\ 
\midrule
\multirow{4}{*}{\textbf{ETTh2}}
& 96  & 0.2854 & 0.3376 & 0.2858 & \textbf{0.3372} 
& 0.3423 & 0.3792 & \textbf{0.3373} & \textbf{0.3748} 
& 0.2924 & 0.3468 & \textbf{0.2906} & \textbf{0.3446} 
& 0.3414 & 0.3953 & \textbf{0.3172} & \textbf{0.3663} \\ 
& 192 & 0.3628 & 0.3892 & \textbf{0.3613} & \textbf{0.3885} 
& 0.4068 & 0.4117 & \textbf{0.4060} & \textbf{0.4109} 
& 0.3794 & 0.4044 & \textbf{0.3761} & \textbf{0.4007} 
& 0.4818 & 0.4792 & \textbf{0.4499} & \textbf{0.4504} \\ 
& 336 & 0.4109 & 0.4222 & 0.4132 & 0.4244 
& 0.4398 & 0.4421 & \textbf{0.4382} & 0.4436 
& 0.4183 & 0.4325 & \textbf{0.4116} & \textbf{0.4290} 
& 0.5929 & 0.5422 & \textbf{0.5773} & \textbf{0.5333} \\ 
& 720 & 0.4485 & 0.4534 & 0.4492 & 0.4561 
& 0.4440 & 0.4564 & 0.4512 & 0.4584 
& 0.4328 & 0.4532 & 0.4392 & 0.4567 
& 0.8403 & 0.6611 & \textbf{0.8089} & \textbf{0.6381} \\ 
\midrule
\multirow{4}{*}{\textbf{ETTm1}}
& 96  & 0.3178 & 0.3563 & \textbf{0.3173} & \textbf{0.3544} 
& 0.3444 & 0.3806 & \textbf{0.3362} & \textbf{0.3770} 
& 0.3268 & 0.3667 & \textbf{0.3209} & \textbf{0.3617} 
& 0.3459 & 0.3737 & \textbf{0.3321} & \textbf{0.3664} \\ 
& 192 & 0.3616 & 0.3829 & \textbf{0.3600} & \textbf{0.3818} 
& 0.4222 & 0.4142 & \textbf{0.4051} & \textbf{0.4100} 
& 0.3718 & 0.3917 & \textbf{0.3694} & \textbf{0.3897} 
& 0.3818 & 0.3911 & \textbf{0.3704} & \textbf{0.3859} \\ 
& 336 & 0.3951 & 0.4067 & 0.3996 & 0.4085 
& 0.4251 & 0.4276 & \textbf{0.4242} & \textbf{0.4269} 
& 0.3981 & 0.4080 & \textbf{0.3932} & \textbf{0.4050} 
& 0.4153 & 0.4150 & \textbf{0.4043} & \textbf{0.4125} \\ 
& 720 & 0.4524 & 0.4413 & 0.4616 & 0.4480 
& 0.4879 & 0.4610 & 0.4911 & 0.4640 
& 0.4586 & 0.4451 & 0.4618 & 0.4513 
& 0.4730 & 0.4509 & \textbf{0.4672} & 0.4553 \\ 
\midrule
\multirow{4}{*}{\textbf{ETTm2}}
& 96  & 0.1706 & 0.2555 & \textbf{0.1701} & \textbf{0.2553} 
& 0.1858 & 0.2667 & \textbf{0.1842} & 0.2672 
& 0.1868 & 0.2692 & \textbf{0.1849} & \textbf{0.2675} 
& 0.1934 & 0.2928 & \textbf{0.1836} & \textbf{0.2744} \\ 
& 192 & 0.2365 & 0.2987 & \textbf{0.2300} & \textbf{0.2961} 
& 0.2532 & 0.3100 & \textbf{0.2486} & \textbf{0.3095} 
& 0.2482 & 0.3070 & \textbf{0.2453} & \textbf{0.3046} 
& 0.2845 & 0.3614 & \textbf{0.2647} & \textbf{0.3364} \\ 
& 336 & 0.2954 & 0.3380 & \textbf{0.2874} & \textbf{0.3365} 
& 0.3237 & 0.3496 & 0.3240 & 0.3502 
& 0.3109 & 0.3492 & \textbf{0.3069} & \textbf{0.3476} 
& 0.3849 & 0.4295 & \textbf{0.3479} & \textbf{0.3989} \\ 
& 720 & 0.3916 & 0.3935 & \textbf{0.3876} & \textbf{0.3921} 
& 0.4162 & 0.4050 & \textbf{0.4102} & \textbf{0.4035} 
& 0.4233 & 0.4157 & \textbf{0.4182} & \textbf{0.4123} 
& 0.5561 & 0.5235 & \textbf{0.5537} & \textbf{0.5137} \\ 
\midrule
\multirow{4}{*}{\textbf{Weather}}
& 96  & 0.1574 & 0.2047 & \textbf{0.1545} & \textbf{0.2018} 
& 0.1684 & 0.2185 & \textbf{0.1622} & \textbf{0.2120} 
& 0.1752 & 0.2174 & \textbf{0.1616} & \textbf{0.2071} 
& 0.1962 & 0.2561 & \textbf{0.1635} & \textbf{0.2158} \\ 
& 192 & 0.2041 & 0.2475 & \textbf{0.2008} & \textbf{0.2450} 
& 0.2345 & 0.2735 & \textbf{0.2069} & \textbf{0.2538} 
& 0.2214 & 0.2562 & \textbf{0.2043} & \textbf{0.2465} 
& 0.2389 & 0.2992 & \textbf{0.2086} & \textbf{0.2631} \\ 
& 336 & 0.2610 & 0.2899 & \textbf{0.2530} & \textbf{0.2852} 
& 0.2902 & 0.3095 & \textbf{0.2682} & \textbf{0.2965} 
& 0.2801 & 0.2975 & \textbf{0.2531} & \textbf{0.2847} 
& 0.2811 & 0.3306 & \textbf{0.2543} & \textbf{0.3004} \\ 
& 720 & 0.3400 & 0.3406 & \textbf{0.3189} & \textbf{0.3283} 
& 0.3593 & 0.3539 & \textbf{0.3230} & \textbf{0.3360} 
& 0.3560 & 0.3475 & \textbf{0.3257} & \textbf{0.3319} 
& 0.3454 & 0.3819 & \textbf{0.3277} & \textbf{0.3540} \\ 
\midrule
\multirow{4}{*}{\textbf{Electricity}}
& 96  & 0.1406 & 0.2426 & \textbf{0.1345} & \textbf{0.2326} 
& 0.1676 & 0.2702 & \textbf{0.1564} & \textbf{0.2591} 
& 0.1802 & 0.2734 & \textbf{0.1509} & \textbf{0.2481} 
& 0.2102 & 0.3014 & \textbf{0.1516} & \textbf{0.2517} \\ 
& 192 & 0.1572 & 0.2558 & \textbf{0.1517} & \textbf{0.2485} 
& 0.1897 & 0.2914 & \textbf{0.1762} & \textbf{0.2787} 
& 0.1875 & 0.2797 & \textbf{0.1639} & \textbf{0.2586} 
& 0.1938 & 0.2799 & \textbf{0.1722} & \textbf{0.2649} \\ 
& 336 & 0.1759 & 0.2754 & \textbf{0.1682} & \textbf{0.2666} 
& 0.1981 & 0.2988 & \textbf{0.1898} & \textbf{0.2916} 
& 0.2042 & 0.2959 & \textbf{0.1800} & \textbf{0.2762} 
& 0.2230 & 0.3191 & \textbf{0.1833} & \textbf{0.2836} \\ 
& 720 & 0.2161 & 0.3091 & \textbf{0.1990} & \textbf{0.2957} 
& 0.2313 & 0.3249 & \textbf{0.2232} & \textbf{0.3213} 
& 0.2435 & 0.3281 & \textbf{0.2142} & \textbf{0.3060} 
& 0.2440 & 0.3323 & \textbf{0.2132} & \textbf{0.3099} \\ 
\midrule
\multirow{4}{*}{\textbf{Traffic}}
& 96  & 0.4277 & 0.2715 & \textbf{0.4221} & \textbf{0.2641} 
& 0.5913 & 0.3160 & \textbf{0.5571} & \textbf{0.2972} 
& 0.4583 & 0.2985 & \textbf{0.4298} & \textbf{0.2808} 
& 0.6965 & 0.4287 & \textbf{0.4889} & \textbf{0.3084} \\ 
& 192 & 0.4587 & 0.2761 & \textbf{0.4558} & \textbf{0.2698} 
& 0.6120 & 0.3278 & \textbf{0.5847} & \textbf{0.3059} 
& 0.4804 & 0.3104 & \textbf{0.4543} & \textbf{0.2919} 
& 0.6014 & 0.3747 & \textbf{0.4839} & \textbf{0.3083} \\ 
& 336 & 0.4762 & 0.2925 & \textbf{0.4699} & \textbf{0.2847} 
& 0.6543 & 0.3484 & \textbf{0.6192} & \textbf{0.3183} 
& 0.4963 & 0.3174 & \textbf{0.4698} & \textbf{0.2990} 
& 0.6077 & 0.3769 & \textbf{0.5036} & \textbf{0.3187} \\ 
& 720 & 0.5195 & 0.3101 & \textbf{0.5135} & \textbf{0.3031} 
& 0.6711 & 0.3490 & \textbf{0.6546} & \textbf{0.3339} 
& 0.5308 & 0.3341 & \textbf{0.5044} & \textbf{0.3174} 
& 0.6482 & 0.3982 & \textbf{0.5404} & \textbf{0.3383} \\ 
\bottomrule
\end{tabular}}
\end{minipage}

%% file: table/pir_comparison_full_one_table.tex

\begin{samepage}
\begin{center}
\captionof{table}{Relative performance of \textbf{PIR (Post-forecasting Identification and Revision)} vs.\ \textbf{CRC}
($\Delta=\mathrm{PIR}-\mathrm{CRC}$) across datasets, horizons, and backbones. Negative values indicate PIR better;
positive values indicate CRC better.}
\label{tab:pir_comparison_full_one_table}
\scriptsize
\setlength{\tabcolsep}{2.2pt}
\renewcommand{\arraystretch}{1.08}

\begin{tabular}{@{}cc*{8}{S}@{}}
\toprule
\textbf{Dataset} & \textbf{Horizon} &
\multicolumn{2}{c}{\textbf{TimeXer}} &
\multicolumn{2}{c}{\textbf{TimesNet}} &
\multicolumn{2}{c}{\textbf{PatchTST}} &
\multicolumn{2}{c}{\textbf{DLinear}} \\
& &
{\scriptsize $\Delta$MSE} & {\scriptsize $\Delta$MAE} &
{\scriptsize $\Delta$MSE} & {\scriptsize $\Delta$MAE} &
{\scriptsize $\Delta$MSE} & {\scriptsize $\Delta$MAE} &
{\scriptsize $\Delta$MSE} & {\scriptsize $\Delta$MAE} \\
\midrule

\multirow{4}{*}{\textbf{ETTh1}}
& 96  & -0.0030 & -0.0019 & -0.0076 & -0.0002 & -0.0033 &  0.0012 & -0.0143 & -0.0159 \\
& 192 & -0.0060 & -0.0041 &  0.0057 &  0.0086 &  0.0072 &  0.0038 & -0.0152 & -0.0227 \\
& 336 & -0.0264 & -0.0213 &  0.0150 &  0.0104 &  0.0077 & -0.0011 & -0.0100 & -0.0211 \\
& 720 & -0.0537 & -0.0391 & -0.0130 & -0.0019 & -0.0260 & -0.0245 & -0.0472 & -0.0485 \\
\addlinespace[2pt]\midrule

\multirow{4}{*}{\textbf{ETTh2}}
& 96  &  0.0092 &  0.0048 &  0.0097 &  0.0012 &  0.0134 &  0.0064 & -0.0312 & -0.0313 \\
& 192 &  0.0057 &  0.0045 &  0.0160 &  0.0072 &  0.0059 &  0.0013 & -0.0819 & -0.0634 \\
& 336 &  0.0038 &  0.0016 &  0.0118 &  0.0024 &  0.0464 &  0.0210 & -0.1643 & -0.1093 \\
& 720 & -0.0492 & -0.0311 & -0.0162 & -0.0114 & -0.0092 & -0.0067 & -0.3919 & -0.2011 \\
\addlinespace[2pt]\midrule

\multirow{4}{*}{\textbf{ETTm1}}
& 96  &  0.0017 &  0.0016 &  0.0278 &  0.0050 &  0.0031 & -0.0007 & -0.0161 & -0.0134 \\
& 192 &  0.0070 &  0.0042 & -0.0251 & -0.0130 & -0.0004 & -0.0017 & -0.0034 & -0.0049 \\
& 336 &  0.0244 &  0.0155 &  0.0198 &  0.0101 &  0.0108 &  0.0080 & -0.0053 & -0.0065 \\
& 720 &  0.0134 &  0.0100 &  0.0239 &  0.0100 &  0.0032 &  0.0017 & -0.0012 & -0.0133 \\
\addlinespace[2pt]\midrule

\multirow{4}{*}{\textbf{ETTm2}}
& 96  &  0.0089 &  0.0057 &  0.0018 &  0.0018 &  0.0101 &  0.0075 & -0.0056 & -0.0104 \\
& 192 &  0.0100 &  0.0079 &  0.0234 &  0.0095 &  0.0097 &  0.0134 & -0.0217 & -0.0304 \\
& 336 &  0.0196 &  0.0065 &  0.0170 &  0.0088 &  0.0081 &  0.0034 & -0.0399 & -0.0529 \\
& 720 &  0.0434 &  0.0249 &  0.0128 &  0.0095 &  0.0128 &  0.0107 & -0.1477 & -0.1087 \\
\addlinespace[2pt]\midrule

\multirow{4}{*}{\textbf{Weather}}
& 96  &  0.0025 &  0.0032 &  0.0058 &  0.0060 &  0.0104 &  0.0069 &  0.0185 &  0.0112 \\
& 192 &  0.0052 &  0.0040 &  0.0141 &  0.0072 &  0.0147 &  0.0105 &  0.0194 & -0.0001 \\
& 336 &  0.0110 &  0.0078 &  0.0148 &  0.0085 &  0.0249 &  0.0133 &  0.0287 &  0.0016 \\
& 720 &  0.0281 &  0.0197 &  0.0420 &  0.0250 &  0.0983 &  0.0461 &  0.0403 &  0.0030 \\
\addlinespace[2pt]\midrule

\multirow{4}{*}{\textbf{Electricity}}
& 96  &  0.0125 &  0.0084 &  0.0016 &  0.0019 &  0.0151 &  0.0009 &  0.0174 & -0.0017 \\
& 192 &  0.0083 &  0.0045 &  0.0038 &  0.0003 &  0.0121 &  0.0014 &  0.0068 & -0.0039 \\
& 336 &  0.0058 &  0.0034 & -0.0018 & -0.0046 &  0.0120 & -0.0012 &  0.0127 & -0.0056 \\
& 720 &  0.0000 & -0.0007 & -0.0042 & -0.0053 &  0.0098 & -0.0010 &  0.0248 &  0.0041 \\
\addlinespace[2pt]\midrule

\multirow{4}{*}{\textbf{Traffic}}
& 96  &  0.0139 &  0.0059 & -0.0441 &  0.0068 &  0.0202 &  0.0022 &  0.0201 &  0.0016 \\
& 192 &  0.0072 &  0.0082 & -0.0447 &  0.0081 &  0.0127 & -0.0049 &  0.0241 &  0.0057 \\
& 336 &  0.0101 &  0.0003 & -0.0622 &  0.0127 &  0.0142 & -0.0060 &  0.0224 & -0.0047 \\
& 720 &  0.0185 &  0.0019 & -0.0566 &  0.0071 &  0.0096 & -0.0064 &  0.0236 & -0.0073 \\
\bottomrule
\end{tabular}
\end{center}
\end{samepage}